\documentclass[11pt]{article}
\usepackage{amsmath, amssymb, amsthm, fullpage}

\title{Why Most Optimism Bandit Algorithms Have the Same Regret Analysis: A Simple Unifying Theorem\thanks{This research was funded by National Science Foundation  grants CCF-2112457 and  CCF-2312198, and Army Research office  grant W911NF-24-1-0083
}}
\author{Vikram Krishnamurthy, Cornell University, Ithaca NY 14853, USA  vikramk@cornell.edu}
\date{\today}

\newtheorem{theorem}{Theorem}
\newtheorem{lemma}{Lemma}
\theoremstyle{remark}
\newtheorem{remark}{Remark}
\newtheorem{condition}{Condition}
\usepackage{enumitem}
\setlist[itemize,enumerate,description]{%
  topsep=0pt,        
  partopsep=0pt,     
  parsep=0pt,        
  itemsep=0pt,       
  leftmargin=*,      
  labelsep=0.5em     
}

\usepackage{geometry}
\begin{document}
\maketitle

\begin{abstract}
Several optimism-based stochastic bandit algorithms -- including UCB, UCB-V, linear UCB, and finite-arm GP-UCB -- achieve logarithmic regret using proofs that, despite superficial differences, follow essentially the same structure. This note isolates the minimal ingredients behind these analyses: a single high-probability concentration condition on the estimators, after which logarithmic regret follows from two short deterministic lemmas describing radius collapse and optimism-forced deviations. The framework  yields unified, near-minimal proofs for these classical algorithms and extends naturally to many contemporary bandit variants.
\end{abstract}

{\bf Keywords}. multi-armed bandits, upper confidence bound, regret analysis, optimism, concentration inequalities, stochastic bandits, linear bandits, unified proofs

\section{Introduction}
We consider a $K$-armed stochastic bandit problem with unknown mean rewards
$\mu_1,\dots,\mu_K$.  Each arm~$i\in [K]$ has mean reward
$\mu_i$ and suboptimality gap $\Delta_i = \mu^\star - \mu_i$. Here  $\mu^\star=\max_{i\in [K]} \mu_i$ and $i^\star \in \arg\max_{i\in [K]} \mu_i$ denotes an optimal arm.   At each time $t=1,2,\ldots$, only one arm can be pulled; 
if  arm $i $ is pulled, then 
the algorithm observes a random reward (sample) $X_{i,t} $ with mean $\mu_i$.  We assume that all pulls of arm~$i$ yield  independent and
identically distributed random variables.

The aim is to construct a sequential arm-selection algorithm that decides which
arm  $A_t\in [K]$  to pull at each time $t$, balancing exploration and exploitation so as to
perform nearly as well as always pulling an optimal arm
$i^\star$.
The suboptimality of such a bandit algorithm is measured by its expected  cumulative regret over a
horizon $T$, defined as
\[
\mathbb{E}[R_T]
=
\mathbb{E}\!\left[\sum_{t=1}^T \bigl(\mu^\star - \mu_{A_t}\bigr)\right] =
\sum_{i \in [K]} \Delta_i  \,\mathbb{E}\{N_i(T)\},
\]
where  $N_i(t)$ is the number of pulls of arm $i$ up to (and including) time $t$.

Let $\widehat{\mu}_i(N_i(t))$ denote  an  estimator of the
mean reward $\mu_i$ constructed from the $N_i(t)$ random rewards obtained by
pulling arm $i$, namely samples
$\{ X_{A_s,s} : A_s = i,\ s \le t \}$.
In optimism-based  bandit algorithms such as UCB, UCB-V, LinUCB and GP-UCB, at each time $t$, the  arm to be pulled is selected
optimistically as (assume each arm is pulled once initially so  $A_1$ is well defined)
\begin{equation}
  \label{eq:bandit_alg}
A_t \in \arg\max_{i\in [K]}\; 
\Bigl[
\widehat{\mu}_{i}(N_i(t-1))
\;+\;
r_i(N_i(t-1))
\Bigr].
\end{equation}
The non-negative function $r_i(\cdot)$ is an algorithm-specific (possibly data-dependent)  \textit{confidence radius}.

\paragraph{Overview and Scope}
This note presents a simple unified  framework for deriving  logarithmic 
regret bounds
for a broad class of bandit algorithms (UCB,
UCB-V, linear UCB, and GP-UCB) in finite-action settings  using essentially the same proof.  The analysis reduces regret bounds to a single high-probability concentration
condition, after which the argument is purely deterministic. We emphasize that the framework does not apply to adversarial bandits,
contextual bandits (where suboptimality gaps depend on the observed context),
instance-independent regret bounds, or continuous-action GP bandits relying on
information-gain arguments. This work was originally motivated by the author's attempt, while writing Chapter 19 of \cite{Kri25} to identify unifying themes across discrete stochastic optimization problems under uncertainty.


\section{Main Condition: Confidence and  Concentration}

We analyze the expected regret accumulated over a fixed horizon $T$; the bandit algorithm
itself need not know $T$ in advance, and all uses of $T$ below are for
analysis purposes only.

\begin{condition}[Concentration Inequality]
  \label{cond:concentration} 
Fix a confidence level $\delta\in(0,1)$.
The confidence radius  $r_i(m)$ used in
\eqref{eq:bandit_alg} satisfies, for all arms $i$, the following concentration
inequality: 
\[
\mathbb P\!\left( \forall m \in \{1,\ldots,  T\} \colon
\bigl|\widehat{\mu}_i(m)-\mu_i\bigr|
\le r_i(m)
\right)\ge 1-\delta,
\]
where for arm dependent  constants $c_1,\sigma_i^2>0$,
\begin{equation}
r_i(m)
=
\sqrt{\frac{2\sigma_i^2 \log(1/\delta)}{m}}
+
\frac{c_1 \log(1/\delta)}{m}.
\label{eq:radius}
\end{equation}
\end{condition}

\begin{remark}
Condition~\ref{cond:concentration} is equivalent to the uniform deviation bound
\[
\mathbb P\!\left(
\exists\, m\in \{1,\ldots, T\}:
\bigl|\widehat{\mu}_i(m)-\mu_i\bigr|> r_i(m)
\right)\le \delta.
\]
\end{remark}

\begin{remark}
The functional in \eqref{eq:radius} is the canonical output of
finite-sample concentration inequalities for averages.
The  term
$\sqrt{(2\sigma_i^2\log(1/\delta))/m}$ reflects the typical
$1/\sqrt{m}$ decay of statistical uncertainty (a variance proxy times a
$\log(1/\delta)$ tail factor), as in Hoeffding or sub-Gaussian bounds.
The second term $(c_1\log(1/\delta))/m$ is a lower-order correction that
appears in sharper Bernstein/Freedman-type inequalities (and more generally
self-normalized bounds), especially when rewards are bounded or modeled as
martingale differences.
Most UCB-style indices can be written with a confidence radius that is either
exactly of this form or is upper bounded by \eqref{eq:radius} up to constants.
\end{remark}

\begin{remark}
Condition~\ref{cond:concentration} holds for empirical mean estimators under
bounded or sub-Gaussian rewards (via Hoeffding, Bernstein, or Freedman
inequalities), and for ridge regression or posterior mean estimators used in
structured bandits such as linear and GP bandits.
In these settings, our regret analysis applies directly when the action set is
finite and repeated pulls of a fixed arm lead to shrinking confidence radii, as
illustrated in the examples below. We interpret  $\sigma_i^2$ as a variance proxy (or scale parameter) associated with arm $i$,
arising from the underlying concentration inequality.
\end{remark}

\begin{remark} 
In  linear and GP bandits, concentration inequalities
are indexed by time $t$ rather than pull count $m$.
Then we interpret $r_i(m)$ as any deterministic bound satisfying
$r_t(i)\le r_i(N_i(t-1))$, and assume that repeated pulls of arm $i$
imply $r_i(m)=O(\sqrt{\log T/m})$.
\end{remark}

An immediate consequence of the above concentration condition is radius collapse.

\begin{lemma}[Radius collapse]
\label{lem:collapse}
Under Condition~\ref{cond:concentration} and  confidence level $\delta$, for each suboptimal arm $i$ with gap $\Delta_i$, there exists an integer
\[
m_0
=
O\!\left(
\frac{\log (1/\delta)}{\Delta_i^2}
+
\frac{\log (1/\delta)}{\Delta_i}
\right)
\]
such that
$
r_i(m)\le \Delta_i/4$ for all $m\ge m_0$.
\end{lemma}

\begin{proof}
Choose $m_0$ large enough so that
\[
\sqrt{\frac{2\sigma_i^2 \log (1/\delta)}{m_0}} \le \frac{\Delta_i}{8}
\qquad\text{and}\qquad
\frac{c_1 \log (1/\delta)}{m_0}\le \frac{\Delta_i}{8}.
\]
These inequalities hold whenever
\begin{equation}
m_0
=
\left\lceil
\max\!\left\{
\frac{128\,\sigma_i^2 \log( 1/\delta)}{\Delta_i^2},
\frac{8c_1 \log (1/\delta)}{\Delta_i}
\right\}
\right\rceil
\label{eq:m0}
\end{equation}
ensures $r_i(m_0)\le \Delta_i/4$. Since $m\mapsto r_i(m)$ is decreasing in~\eqref{eq:radius}, 
the same $m_0$ implies $r_i(m)\le \Delta_i/4$ for all $m\ge m_0$. This yields
$m_0=O(\log (1/\delta)/\Delta_i^2+\log (1/\delta)/\Delta_i)$, as claimed.
\end{proof}


\section{Meta-Result. Logarithmic Regret}

Recall  that $N_i(t):=\sum_{s=1}^t \mathbf 1\{A_s=i\}$ is the number of pulls of arm $i$
up to (and including) time~$t$. 
For each arm $i$, let $\widehat{\mu}_i(m)$ be an estimator
of $\mu_i$ based on $m$ i.i.d.\ samples from arm $i$.
Let $i^\star \in \arg\max_j \mu_j$ denote an optimal arm, and let
$\mu^\star = \mu_{i^\star}$ denote the optimal mean reward.

\begin{lemma}[Optimism forces a deviation]
  \label{lem:optimism}
  Consider algorithm~\eqref{eq:bandit_alg}.
  Fix a suboptimal arm $i$ with gap $\Delta_i=\mu^\star-\mu_i>0$, and suppose Condition~\ref{cond:concentration} holds.
Then for every $m\ge m_0+1$,
\begin{equation}
\label{eq:optimism-deviation}
\{N_i(T)\ge m\}
\Rightarrow
\Big\{
|\widehat{\mu}_i(m-1)-\mu_i|\ge \Delta_i/4
\ \ \text{or}\ \ 
\exists n\le T:\ |\widehat{\mu}_{i^\star}(n)-\mu^\star|\ge \Delta_i/4
\Big\}.
\end{equation}
\end{lemma}

\begin{proof}
Assume $N_i(T)\ge m$ and let $t$ be the time of the $m$-th pull of arm $i$.
Then $N_i(t-1)=m-1$ and, by optimism,
\[
\widehat{\mu}_i(m-1)+r_i(m-1)
\ge
\widehat{\mu}_{i^\star}(N_{i^\star}(t-1))+r_{i^\star}(N_{i^\star}(t-1)).
\]
If simultaneously
$|\widehat{\mu}_i(m-1)-\mu_i|<\Delta_i/4$ and
$|\widehat{\mu}_{i^\star}(N_{i^\star}(t-1))-\mu^\star|<\Delta_i/4$,
then using $r_i(m-1)\le \Delta_i/4$ (since $m-1\ge m_0$),
\[
\widehat{\mu}_i(m-1)+r_i(m-1)
\le \mu_i+\Delta_i/4+\Delta_i/4
= \mu^\star-\Delta_i/2
<
\mu^\star-\Delta_i/4
\le
\widehat{\mu}_{i^\star}(N_{i^\star}(t-1))+r_{i^\star}(N_{i^\star}(t-1)),
\]
a contradiction. Hence at least one deviation in \eqref{eq:optimism-deviation}
must occur. Since $N_{i^\star}(t-1)\le T$, the second deviation implies the
stated $\exists n\le T$ event.
\end{proof}

Optimism implicitly governs how the algorithm transitions between arms.
By Lemma~\ref{lem:collapse}, the confidence radius of a suboptimal arm $i$
shrinks with the number of pulls and eventually becomes smaller than its gap
$\Delta_i$.
Once this occurs, Lemma~\ref{lem:optimism} shows that selecting arm $i$ again
would force a violation of at least one confidence bound -- either for arm $i$
itself or for the optimal arm. 
On the high-probability event where all confidence bounds hold, such deviations
are impossible, and hence the algorithm cannot continue selecting arm $i$.
Consequently, each suboptimal arm is explored only until its uncertainty collapses
below its gap, after which optimism automatically shifts the algorithm’s attention
to other arms whose confidence radii remain large.

\begin{theorem}[Logarithmic regret]
\label{thm:master}
Consider algorithm~\eqref{eq:bandit_alg} and suppose
Condition~\ref{cond:concentration} holds.
Fix  $\delta=1 /(KT)$.
Then for every suboptimal arm $i$ with gap $\Delta_i>0$,
\[
\mathbb{E}N_i(T)\le m_0+ 1 ,
\]
and consequently
\[
\mathbb{E}R_T
=
O\!\left(
\sum_{i:\Delta_i>0}\frac{\log T}{\Delta_i}
\right),
\]
where the $O(\cdot)$ hides universal constants and $\log K$ factors.
\end{theorem}

\begin{proof}
Define the \emph{good event} on which all confidence bounds hold uniformly as
\[
\mathcal{E}
=
\Big\{
|\widehat{\mu}_j(m)-\mu_j|\le r_j(m)
\ \ \forall j\in[K],\ \forall m\le T
\Big\}.
\]
By Condition~\ref{cond:concentration} and a union bound over
$j\in[K]$, $\mathbb{P}(\mathcal{E})\ge 1-K \delta  = 1 -1 /T$.

On $\mathcal{E}$, Lemma~\ref{lem:optimism} implies that for any
$m\ge m_0+1$, the event $\{N_i(T)\ge m\}$ is impossible, since
$r_i(m-1)\le r_i(m_0)\le \Delta_i/4$ (because $m-1\ge m_0$) and all deviation events are ruled
out.
Hence $N_i(T)\le m_0$ on $\mathcal{E}$, and
\[
\mathbb{E}N_i(T) = \mathbb{E}\big[ N_i(T) \mathbf{1}_{\mathcal{E}} \big] +   \mathbb{E}\big[ N_i(T) \mathbf{1}_{\mathcal{E}^c} \big]
\le
m_0 + T\,\mathbb{P}(\mathcal{E}^c)
\le
m_0+ 1 .
\]
Since $R_T=\sum_{i:\Delta_i>0}\Delta_i N_i(T)$ and on the good event $\mathcal{E}$ we have
$N_i(T)\le m_0=O(\log T/\Delta_i^2)$ for each suboptimal arm (see~\eqref{eq:m0}),
it follows that
$
\mathbb{E}R_T
=
O\!\left(\sum_{i:\Delta_i>0}\frac{\log T}{\Delta_i}\right)$.
\end{proof}

\section{Examples}

\paragraph{Example 1: UCB (Bounded Rewards)}

Consider the classical UCB index \cite{Auer02}
\[
U_i(m) = \widehat{\mu}_i(m) 
+ \sqrt{\frac{2\log T}{m}},
\]
where $\widehat{\mu}_i(m)$ is the empirical mean of arm $i$ after
$m$ pulls and the observed random rewards are $X_{i,t}\in[0,1]$.
By Hoeffding, Condition~\ref{cond:concentration}  holds with $\sigma_i^2 = 1/4$.
Hence Theorem~\ref{thm:master}  holds.

\begin{remark}
In UCB, one can use either $\sqrt{(2\log t)/m}$ or
$\sqrt{(2\log T)/m}$; this does not change the regret order.
Indeed, for a suboptimal arm $i$, one needs a
sample size $m_0$ such that the confidence radius is at most a fixed
fraction of the gap, e.g.
$
\sqrt{\frac{2\log(\cdot)}{m_0}} \le \Delta_i/4.
$ which yields 
$
m_0 = C\,\frac{\log T}{\Delta_i^2}
$
for some constant $C$, regardless of whether the logarithm is
$\log t$ or $\log T$.
\end{remark}

\paragraph{Example 2: UCB-V (Bounded Variance)}

Assume $X_{i,s} \in [0,1]$ and $\mathrm{Var}(X_{i,s}) \le \sigma_i^2 \le 1/4$.  
Let $\widehat{\mu}_i(m)$ and $s_i^2(m)$ denote the empirical mean and variance.
The UCB-V index is \cite{Audibert09}
\[
U_i^V(m)
=
\widehat{\mu}_i(m) + r_i(m), \quad \text{ where } r_i(m) = 
 \sqrt{\frac{2 s_i^2(m)\log T}{m}}
+ \frac{3\log T}{m}.
\]
The  Freedman/Bernstein self-normalized bound implies  Condition~\ref{cond:concentration}  for
$\widehat{\mu}_i(m)$ with variance proxy $\sigma_i^2$.
Hence Theorem~\ref{thm:master}  holds.

\paragraph{Example 3: Linear UCB (Sketch)}
Each arm has feature $x_i\in\mathbb{R}^d$, and rewards satisfy
\cite{AbbasiYadkori11}
\[
X_{i,s} = x_i^\top\theta_\star + \eta_{i,s},
\]
where $\eta_{i,s}$ is  sub-Gaussian noise.  Let
\[
\widehat{\theta}_t = V_t^{-1} \sum_{s<t} x_{A_s} X_{A_s,s},
\qquad
V_t = \lambda I + \sum_{s<t} x_{A_s}x_{A_s}^\top.
\]
The linear UCB index is
\[
U_i^{\mathrm{lin}}(t)
=
x_i^\top \widehat{\theta}_t + r_t(i) , \quad \text{ where }  r_t(i) = 
 \alpha_t \sqrt{x_i^\top V_t^{-1}x_i}.
\]
The self-normalized concentration inequality 
(Abbasi-Yadkori et al., 2011) gives, for all $i,t$,
\[
\mathbb{P}\Big(
 |x_i^\top(\widehat{\theta}_t - \theta_\star)|
 \le r_t(i)
\Big)\ge 1-\delta,
\]
which plays the role of Condition~\ref{cond:concentration} for the predicted rewards
$x_i^\top\theta_\star$.

In the finite-action setting, repeated pulls of arm $i$ imply that
$r_t(i)$ decreases as a function of $m=N_i(t-1)$; in particular, we may upper-bound
$r_t(i)\le r_i(m)$ with $r_i(m)=O(\sqrt{\log T/m})$.
Lemma~\ref{lem:optimism} then implies that selecting arm $i$ at visit $m\ge m_0$
forces a violation of the corresponding confidence bound of order $\Delta_i$,
and hence arm $i$ can be selected only $O(\log T/\Delta_i^2)$ times.

\paragraph{Example 4: GP-UCB (Finite-Arm Sketch)}

Let the unknown function $f$ lie in an RKHS $\mathcal{H}_k$ with
$\|f\|_{\mathcal{H}_k}\le B$, and let the finite set of arms be
$\{x_1,\dots,x_K\}$.
Observations satisfy
\[
X_{i,t} = f(x_i) + \eta_t,
\]
with sub-Gaussian noise.
Let $\widehat{f}_{t-1}$ be the kernel ridge regression estimator and
$\sigma_{t-1}(x_i)$ the posterior predictive standard deviation.

GP-UCB uses the index \cite{Srinivas10}
\[
U_i^{\mathrm{GP}}(t)
=
\widehat{f}_{t-1}(x_i)
+ r_t(i), \quad  \text{ where } r_t(i) = \beta_t^{1/2}\,\sigma_{t-1}(x_i).
\]
For finite arms,  a union bound over Gaussian tail inequalities 
implies the uniform concentration
\[
\mathbb{P}\Big(
 |f(x_i)-\widehat{f}_{t-1}(x_i)| \le r_t(i)
\Big)\ge 1-\delta,
\]
which plays the role of Condition~\ref{cond:concentration} for the predicted
rewards $f(x_i)$.

Fix a suboptimal arm $i$ with gap $\Delta_i = f(x^\star)-f(x_i)$.
After
$
m_0 \;\ge\; C\,\frac{\log T}{\Delta_i^2}
$
visits, for constant $C> 0$, the posterior variance satisfies
$
r_t(i) = \beta_t^{1/2}\sigma_{t-1}(x_i)
\le \Delta_i/8$.
Lemma~\ref{lem:optimism} then implies that selecting arm $i$ at visit $m\ge m_0$
forces a violation of at least one confidence bound -- either for arm $i$ or for
the optimal arm  -- of order $\Delta_i$.
Thus the radius-collapse condition required by Lemma~\ref{lem:collapse} holds,
and Theorem~\ref{thm:master} applies.


\paragraph{Example 5: Bandits with Heteroskedastic Noise (Sketch).}
Heteroskedasticity   refers arm-dependent i.i.d.\  noise, and 
not to sophisticated time-series models such as  GARCH.
Consider a stochastic bandit in which rewards from arm $i$ satisfy
\cite{Kirschner18}
\[
X_{i,t} = \mu_i + \sigma_i\,\eta_{i,t},
\]
where $\eta_{i,t}$ is i.i.d.\ sub-Gaussian across pulls of arm $i$, and the noise
scale $\sigma_i^2$ is unknown and may vary across arms.
Several recent algorithms construct data-dependent confidence radii using
empirical Bernstein or self-normalized concentration inequalities to account
for heteroskedastic noise.

Such bounds yield confidence radii of the form
\[
\mathbb{P}\big(
\forall \, m\le T:\ 
\bigl|\widehat{\mu}_i(m)-\mu_i\bigr|
\leq r_i(m) \big) \geq 1-\delta, \quad
\text{ where } r_i(m) = \sqrt{\frac{2\,\widehat{\sigma}_i^2(m)\,\log(1/\delta)}{m}}
+
\frac{c\,\log(1/\delta)}{m}
.
\]
Here $\widehat{\sigma}_i^2(m)$ denotes the empirical variance of  $m$
observations from arm $i$.

Once this condition is verified, the regret analysis follows immediately
from Lemmas~\ref{lem:collapse} and~\ref{lem:optimism}, yielding logarithmic
regret with arm-dependent constants.

\paragraph{Example 6: PAC (Anytime) Confidence Bounds (Sketch).}
Some bandit algorithms are based on PAC-style confidence bounds 
\cite{EvenDar02}.
For each arm $i$ and confidence level $\delta\in(0,1)$,
a PAC  bound refers to the concentration inequality:
there exists a
radius $r_i(m,\delta)$ such that uniformly in~$m$, 
\[
\mathbb{P}\!\left(
\exists\, m\ge 1:\ 
\bigl|\widehat{\mu}_i(m)-\mu_i\bigr|
>
r_i(m,\delta)
\right)
\le \delta.
\]
Such bounds immediately imply Condition~\ref{cond:concentration} by restriction
to $m\le T$.
Consequently, once a PAC-style concentration inequality is available, the regret
analysis in Theorem~\ref{thm:master} applies without modification.
We emphasize that this observation concerns only the regret argument and does not
replace the derivation of PAC confidence bounds.

\paragraph{Example 7: Heavy-Tailed Stochastic Bandits (Sketch).}
Recent work studies stochastic bandits with heavy-tailed rewards, where the
reward distributions may have unbounded support and possibly infinite variance,
but admit robust mean estimators yielding sub-Gaussian-type concentration under
appropriate truncation or finite-variance conditions.
In such settings, classical Hoeffding-type confidence bounds are no longer
applicable.

To address this, algorithms such as the data-driven UCB variants in \cite{Tamas24}
use robust mean estimators, including truncated means or median-of-means
estimators, to construct high-probability confidence bounds without requiring
sub-Gaussian noise assumptions or prior knowledge of tail parameters.
Specifically, robust mean estimators for heavy-tailed rewards (such as truncated
means or median-of-means estimators) satisfy uniform-in-time concentration
inequalities of the form
\[
\mathbb{P}\Big(
\exists\, m\ge 1:\ 
\bigl|\widehat{\mu}_i(m)-\mu_i\bigr|
> r_i(m)
\Big)\le \delta,
\quad
\text{where}\quad
r_i(m)
=
C \sqrt{\frac{\log(1/\delta)}{m}}
+
\frac{D \log(1/\delta)}{m}.
\]
Here  $C,D>0$ are constants depending on the chosen robust estimator and truncation or
moment assumptions, derived from self-normalized concentration inequalities
tailored to heavy-tailed data (e.g., Bernstein-type bounds for truncated
statistics).
Such bounds ensure that Condition~\ref{cond:concentration} holds uniformly
over $m\le T$, even in the absence of sub-Gaussian noise.

Repeated pulls of arm $i$ shrink $r_i(m)$ as $O(\sqrt{\log T/m})$, allowing
Lemma~\ref{lem:collapse} to apply.
Lemma~\ref{lem:optimism} then implies that selecting a suboptimal arm beyond
$m_0=O((\log T)/\Delta_i^2)$ would require a violation of the corresponding
confidence bound.
Consequently, Theorem~\ref{thm:master} yields logarithmic regret, with constants
depending on the heavy-tail assumptions.

\paragraph{Example 8: Misspecified Linear Bandits with Gap-Adjusted Assumptions (Sketch).}
In misspecified linear bandits, the reward model deviates from exact linearity, i.e.,
\[
X_{i,s} = x_i^\top \theta_\star + \epsilon_i + \eta_{i,s},
\]
where $\epsilon_i$ is an arm-dependent misspecification error and
$\eta_{i,s}$ is sub-Gaussian noise.
Standard linear bandit algorithms suffer linear regret under unrestricted
misspecification.

Recent analyses  \cite{Liu25} introduce \emph{gap-adjusted misspecification} assumptions,
where the misspecification error is proportional to the suboptimality gap:
$|\epsilon_i|\le c\,\Delta_i$ for some constant $c<1/2$, with
$\Delta_i=\mu^\star-(x_i^\top\theta_\star+\epsilon_i)$.
Under this condition, optimistic algorithms based on confidence ellipsoids, such as variants of OFUL with phased exploration or G-optimal designs, have
logarithmic regret.

In finite-action settings, repeated pulls of arm $i$ imply that the prediction
radius satisfies $r_t(i)\le r_i(m)$ with
$r_i(m)=O(\sqrt{d\log T/m})$, where $m=N_i(t-1)$.
Self-normalized concentration inequalities ensure that
Condition~\ref{cond:concentration} holds for the biased mean
$x_i^\top\theta_\star+\epsilon_i$.
Choosing $m_0=O((d\log T)/\Delta_i^2)$ such that
$r_i(m_0)\le (1-c)\Delta_i/4$, Lemma~\ref{lem:optimism} implies that selecting
arm $i$ beyond $m_0$ would force a confidence violation.
Thus, Theorem~\ref{thm:master} yields logarithmic regret despite model
misspecification, provided the misspecification is gap-proportional.

\paragraph{Example 9: Bandits with Machine Learning-Assisted Surrogate Rewards (Sketch).}
Recent work studies stochastic bandits augmented with surrogate rewards generated
by machine learning models trained on auxiliary offline data \cite{Ji25}.
These surrogates provide biased but correlated predictions of the true rewards,
enabling variance reduction when combined with appropriate debiasing techniques.
The Machine Learning-Assisted UCB (MLA-UCB) algorithm combines online observations
with surrogate predictions to construct tighter confidence bounds.

Specifically, the debiased estimator used by ML-assisted UCB algorithms satisfies
a uniform-in-time concentration inequality of the form
\[
\mathbb{P}\Big(
\exists\, m\ge 1:\ 
\bigl|\widehat{\mu}_i(m)-\mu_i\bigr|
> r_i(m)
\Big)\le \delta,
\quad
\text{where}\quad
r_i(m)
=
O\!\left(
\sqrt{\frac{\mathrm{Var}_{\mathrm{eff}}(m)\log(1/\delta)}{m}}
+
\frac{\log(1/\delta)}{m}
\right).
\]
Here $\mathrm{Var}_{\mathrm{eff}}(m)$ is an effective variance parameter reduced
by the correlation between the surrogate and true rewards, and depends on the
quality of the debiasing procedure.
Uniform concentration over $m\le T$ can be established via exact distributional
analysis and self-normalized tests for the debiased estimator, ensuring that
Condition~\ref{cond:concentration} holds with improved constants.

Due to the reduced effective variance, repeated pulls of arm $i$ shrink
$r_i(m)$ faster than in standard UCB, allowing Lemma~\ref{lem:collapse} to apply
with a smaller $m_0$ (improved constants).
Lemma~\ref{lem:optimism} then forces a confidence violation for further pulls of
a suboptimal arm beyond this threshold.
Consequently, Theorem~\ref{thm:master} yields logarithmic regret with strictly
improved constants compared to classical UCB, while accommodating biased
surrogate information and unknown reward variances.

\paragraph{Example 10: Stochastic Contextual Linear Bandits with Diverse Contexts (Sketch)}
Consider stochastic linear contextual bandits with finite actions $K$, where each
time $t$ reveals context vectors $\{x_{t,i}\}_{i=1}^K\subset\mathbb{R}^d$ drawn
i.i.d.\ from a distribution satisfying a \emph{diversity assumption}: the expected
covariance $\mathbb{E}[x_{t,i}x_{t,i}^\top]$ is positive definite with minimum
eigenvalue bounded away from zero \cite{Ghosh22}. Rewards satisfy
\[
X_{t,a_t} = x_{t,a_t}^\top \theta_\star + \eta_t,
\]
with sub-Gaussian noise $\eta_t$.  Each context vector satisfies $x_{t,i}\in\mathbb{R}^d$, where $d$ is the feature dimension.

The LR-SCB algorithm employs phased optimism using norm-adaptive linear UCB
(OFUL-style) on shifted rewards. Self-normalized martingale concentration yields
the uniform bound
\[
\mathbb{P}\Big(
\exists\, t\le T,\ \exists\, i\in[K]:
\big|x_{t,i}^\top(\widehat{\theta}_t-\theta_\star)\big|
> r_t(i)
\Big)\le \delta,
\quad
r_t(i)=\alpha_t\sqrt{x_{t,i}^\top V_t^{-1}x_{t,i}},
\]
\[ \text{ where } 
V_t=\lambda I+\sum_{s<t} x_{s,a_s}x_{s,a_s}^\top,
\qquad
\alpha_t = O\!\left(\sqrt{d\log\!\frac{T}{\delta}}\right).
\]
Under the diversity assumption, $V_t\succeq c\,t\,I$ with high probability (for  constant $c> 0$), implying
\[
|x_{t,i}^\top(\widehat{\theta}_t-\theta_\star)|
\le r_t(i)
= O\!\left(\sqrt{\frac{d\log T}{t}}\right),
\]
which plays the role of Condition~\ref{cond:concentration} for the predicted means.

Although contextual gaps vary with $t$ and Theorem~\ref{thm:master} does not apply
verbatim, the same optimism–collapse–deviation mechanism operates at the
\emph{parameter level}.
The confidence ellipsoid for $\theta_\star$ contracts globally, and a
parameter-level analogue of Lemma~\ref{lem:collapse} holds:
after $m_0=O(\mathrm{poly}(d)\log T)$ rounds, the confidence radius becomes smaller
than any fixed \emph{distributional gap} (expected suboptimality under the context
distribution).
A contextual analogue of Lemma~\ref{lem:optimism} then implies that selecting
suboptimal actions beyond this phase would require a violation of the confidence
bounds, which is impossible on the good event.
This explains how diversity assumptions allow stochastic contextual bandits to
achieve polylogarithmic regret, despite time-varying gaps.
Note that Theorem~\ref{thm:master} does not apply to contextual bandits in full generality.

\section{Conclusion and Extensions}

The above framework isolates the essential structure behind several optimism-based
bandit algorithms.
Once a suitable high-probability concentration inequality is available for the
estimators used by the algorithm, logarithmic regret follows from a simple and
largely deterministic argument formalized in Theorem~\ref{thm:master}.
This yields short, unified proofs for UCB, UCB-V, linear bandits,
and GP bandits with finite action sets, and extend to other structured bandit
settings where repeated selection of a fixed arm leads to shrinking
confidence radii.

{\bf Extension to Randomized Indices.}
Consider a randomized index policy
\[
A_t \in \arg\max_{i\in[K]}
\Bigl(\widehat{\mu}_i(N_i(t-1)) + r_i(N_i(t-1)) + \xi_{i,t}\Bigr),
\]
where $\{\xi_{i,t}\}$ are possibly dependent random perturbations.

\begin{condition}[Uniformly shrinking perturbations]
\label{cond:perturb}
There exists a deterministic sequence $\rho_T \downarrow 0$ such that for any
$\delta\in(0,1)$,
\[
\mathbb P\!\left(
\max_{i\in[K],\,t\le T} |\xi_{i,t}| > \rho_T
\right)\le \delta.
\]
\end{condition}

On the event in Condition~\ref{cond:perturb}, the randomized policy coincides
with a deterministic optimistic rule with inflated radius
$r_i(m)+\rho_T$.

\begin{lemma}[Perturbed optimism implies deviation]
Assume Conditions~\ref{cond:concentration}  and
\ref{cond:perturb}. If $\rho_T \le \Delta_i/8$, then for all $m\ge m_0+1$,
\[
\{N_i(T)\ge m\}
\Rightarrow
\Big\{
|\widehat{\mu}_i(m-1)-\mu_i|\ge \Delta_i/8
\ \text{or}\ 
\exists n\le T:\ |\widehat{\mu}_{i^\star}(n)-\mu^\star|\ge \Delta_i/8
\Big\}.
\]
\end{lemma}

The proof follows identically to Lemma~\ref{lem:optimism}, replacing
$r_i(m)$ by $r_i(m)+\rho_T$.
Consequently, Theorem~\ref{thm:master} continues to hold with unchanged
regret order provided $\rho_T = O(\Delta_i)$.

For follow-the-perturbed-leader (FTPL), the perturbations $\xi_{i,t}$ are
explicitly generated and can often be controlled uniformly, so that
Condition~\ref{cond:perturb} applies directly.
For Thompson sampling, the posterior sample can be viewed as a randomized
perturbation of the posterior mean; however, such perturbations do not admit
uniform high-probability bounds of the form required by
Condition~\ref{cond:perturb}.
However, the same structural principle applies: after sufficient samples,
selecting a suboptimal arm requires a violation of the underlying confidence
bounds in the posterior sample.
Regret guarantees for Thompson sampling rely on additional
probabilistic arguments 
which lie beyond the scope of this note.
For IDS, many practical implementations employ randomized posteriors or
resampling schemes; when these randomizations are uniformly controlled,
the above argument applies.


\end{document}